\documentclass{ecai}
\usepackage{times}
\usepackage{graphicx}
\usepackage{latexsym}

\usepackage{changepage}
\usepackage{hyperref}       % hyperlinks
\usepackage{url}            % simple URL typesetting
\usepackage{booktabs}       % professional-quality tables
\usepackage{amsfonts}       % blackboard math symbols
\usepackage{nicefrac}       % compact symbols for 1/2, etc.
\usepackage{microtype}      % microtypography
\usepackage{algorithm}
\usepackage[noend]{algpseudocode}
\usepackage{todonotes}
\usepackage{bm}
\usepackage{multirow}
\usepackage{amsthm}
\usepackage{appendix}
\usepackage{amsmath}        % for \boldsymbol macro

\newtheorem{thm}{Theorem}
\newtheorem{ass}{Theorem}
\newtheorem{lem}{Theorem}
\newtheorem{cor}{Theorem}

\newtheorem{assumption}[ass]{Assumption}
\newtheorem{lemma}[lem]{Lemma}
\newtheorem{corollary}[cor]{Corollary}

\newtheorem{theorem1}[thm]{Theorem}

\begin{document}

\title{Layer-wise Adaptive Gradient Sparsification for Distributed Deep Learning with Convergence Guarantees}

\author{Shaohuai Shi \and Zhenheng Tang \and Qiang Wang \and Kaiyong Zhao \and Xiaowen Chu\institute{High-Performance Machine Learning Lab, Department of Computer Science, Hong Kong Baptist University,
Hong Kong S.A.R. China, email: \{csshshi,zhtang,qiangwang,kyzhao,chxw\}@comp.hkbu.edu.hk} }

\maketitle
\bibliographystyle{ecai}

\begin{abstract}
To reduce the long training time of large deep neural network (DNN) models, distributed synchronous stochastic gradient descent (S-SGD) is commonly used on a cluster of workers. However, the speedup brought by multiple workers is limited by the communication overhead. Two approaches, namely pipelining and gradient sparsification, have been separately proposed to alleviate the impact of communication overheads. Yet, the gradient sparsification methods can only initiate the communication after the backpropagation, and hence miss the pipelining opportunity. In this paper, we propose a new distributed optimization method named LAGS-SGD, which combines S-SGD with a novel layer-wise adaptive gradient sparsification (LAGS) scheme. In LAGS-SGD, every worker selects a small set of ``significant'' gradients from each layer independently whose size can be adaptive to the communication-to-computation ratio of that layer. The layer-wise nature of LAGS-SGD opens the opportunity of overlapping communications with computations, while the adaptive nature of LAGS-SGD makes it flexible to control the communication time. We prove that LAGS-SGD has convergence guarantees and it has the same order of convergence rate as vanilla S-SGD under a weak analytical assumption. Extensive experiments are conducted to verify the analytical assumption and the convergence performance of LAGS-SGD. Experimental results on a 16-GPU cluster show that LAGS-SGD outperforms the original S-SGD and existing sparsified S-SGD without losing obvious model accuracy.
\end{abstract}

\section{INTRODUCTION}
With increasing data volumes and model sizes of deep neural networks (DNNs), distributed training is commonly adopted to accelerate the training process among multiple workers. Current distributed stochastic gradient descent (SGD) approaches can be categorized into three types, synchronous \cite{dekel2012optimal,li2014efficient}, asynchronous \cite{zinkevich2010parallelized} and stall synchronous \cite{ho2013more}. Synchronous SGD (S-SGD) with data-parallelism is the most widely used one in distributed deep learning due to its good convergence properties \cite{dean2012large,goyal2017accurate}. However, S-SGD requires iterative synchronization and communication of dense gradient/parameter aggregation among all the workers. Compared to the computing speed of modern accelerators (e.g., GPUs and TPUs), the network speed is usually slow which makes communications a potential system bottleneck. Even worse, the communication time usually grows with the size of the cluster \cite{you2017scaling}. Many recent studies focus on alleviating the impact of communications in S-SGD to improve the system scalability. These studies include the system-level methods and the algorithm-level methods. 

On the system level, pipelining \cite{awan2017s,zhang2017poseidon,goyal2017accurate,shi2018performance,li2018pipe,harlap2018pipedream,jia2018highly,shi2019mgwfbp} is used to overlap the communications with the computations by exploiting the layer-wise structure of backpropagation during the training process of deep models. On the algorithmic level, researchers have proposed gradient quantization (fewer bits for a number) and sparsification (zero-out gradients that are not necessary to be communicated) techniques for S-SGD to reduce the communication traffic with negligible impact on the model convergence \cite{alistarh2017qsgd,chen2017adacomp,wen2017terngrad,lin2017deep,wu2018error,karimireddy2019error}. The gradient sparsification method is more aggressive than the gradient quantization method in reducing the communication size. For example, Top-$k$ sparsification \cite{aji2017sparse,lin2017deep} with error compensation can zero-out $99\%-99.9\%$ local gradients without loss of accuracy while quantization from $32$-bit floating points to $1$-bit has a maximum of $32\times$ size reduction. In this paper, we mainly focus on the sparsification methods, while our proposed algorithm and analysis are also applicable to the quantization methods.

A number of works have investigated the theoretical convergence properties of the gradient sparsification schemes under different analytical assumptions \cite{wangni2018gradient,stich2018sparsified,alistarh2018convergence,jiang2018linear,ivkin2019communication,karimireddy2019error}. However, these gradient sparsification methods ignore the layer-wise structure of DNNs and treat all model parameters as a single vector to derive the convergence bounds, which implicitly requires a single-layer communication \cite{you2017scaling} at the end of each SGD iteration. Therefore, the current gradient sparsification S-SGD (denoted by SLGS-SGD hereafter) cannot overlap the gradient communications with backpropagation computations, which limits the system scaling efficiency. To tackle this challenge, we propose a new distributed optimization algorithm named LAGS-SGD which exploits a layer-wise adaptive gradient sparsification (LAGS) scheme atop S-SGD to increase the system scalability. We also derive the convergence bounds for LAGS-SGD. Our theoretical convergence results on LAGS-SGD conclude that high compression ratios would slow down the model convergence rate, which indicates that one should choose the compression ratios for different layers as low as possible. The adaptive nature of LAGS-SGD provides flexible options to choose the compression ratios according to the communication-to-computation ratios. We evaluate our proposed algorithm on various DNNs to verify the soundness of the weak analytic assumption and the convergence results. Finally, we demonstrate our system implementation of LAGS-SGD to show the wall-clock training time improvement on a 16-GPU cluster with 10Gbps Ethernet interconnect. The contributions of this work are summarized as follows.

\begin{itemize}
    \item We propose a new distributed optimization algorithm named LAGS-SGD with convergence guarantees. The proposed algorithm enables us to embrace the benefits of both pipelining and gradient sparsification.
    \item We provide thorough convergence analysis for LAGS-SGD on non-convex smooth optimization problems, and the derived theoretical results indicate that LAGS-SGD has a consistent convergence guarantee with SLGS-SGD, and it has the same order of convergence rate with S-SGD under a weak analytical assumption. 
    \item We empirically verify the analytical assumption and the convergence performance of LAGS-SGD on various deep neural networks including CNNs and LSTM in a distributed setting.
    \item We implement LAGS-SGD atop PyTorch\footnote{\url{https://pytorch.org/}}, which is one of the popular deep learning frameworks, and evaluate the training efficiency of LAGS-SGD on a 16-GPU cluster connected with 10Gbps Ethernet. Experimental results show that LAGS-SGD outperforms S-SGD and SLGS-SGD on a 16-GPU cluster with little impact on the model accuracy.
\end{itemize}

The rest of the paper is organized as follows. Section \ref{sec:relatedwork} introduces some related work, and Section \ref{sec:preliminaries} presents  preliminaries for our proposed algorithm and theoretical analysis. We propose the LAGS-SGD algorithm and provide the theoretical results in Section \ref{sec:algo}. The efficient system design for LAGS-SGD is illustrated in Section \ref{sec:system}. Experimental results and discussions are presented in Section \ref{sec:experiments}. Finally, we conclude the paper in Section \ref{sec:conclude}. 

\section{RELATED WORK}\label{sec:relatedwork}
Many recent works have provided convergence analysis for distributed SGD with quantified or sparsified gradients that can be biased or unbiased.

For the unbiased quantified or sparsified gradients, researchers \cite{alistarh2017qsgd,wen2017terngrad} derived the convergence guarantees for lower-bit quantified gradients, while the quantization operator applied on gradients should be unbiased to guarantee the theoretical results. On the gradient sparsification algorithm whose sparsification method is also unbiased, Wangni et al. \cite{wangni2018gradient} derived the similar theoretical results. However, empirical gradient sparsification methods (e.g., Top-$k$ sparsification \cite{lin2017deep}) can be biased, which require some other analytical techniques to derive the bounds. In this paper, we also mainly focus on the bias sparsification operators like Top-$k$ sparsification.

For the biased quantified or sparsified gradients, Cordonnier \cite{cordonnier2018convex} and Stich et al. \cite{stich2018sparsified} provided the convergence bound for top-$k$ or random-$k$ gradient sparsification algorithms on only convex problems. Jiang et al. \cite{jiang2018linear} derived similar theoretical results, but they exploited another strong assumption that requires each worker to select the same $k$ components at each iteration so that the whole $d$ (the dimension of model/gradient) components are exchanged after a certain number of iterations. Alistarh et al. \cite{alistarh2018convergence} relaxed these strong assumptions on sparsified gradients, and further proposed an analytical assumption, in which the $\ell_2$-norm of the difference between the top-$k$ elements on fully aggregated gradients and the aggregated results on local top-$k$ gradients is bounded. Though the assumption is relaxed, it is difficult to verify in real-world applications. Our convergence analysis is relatively close to the study \cite{shi2019aconvergence} which provided convergence analysis on the biased Top-$k$ sparsification with an easy-to-verify analytical assumption. 

The above mentioned studies, however, view all the model parameters (or gradients) as a single vector to derive the convergence bounds, while we propose the layer-wise gradient sparsification algorithm which breaks down full gradients into multiple pieces (i.e., multiple layers). It is obvious that breaking a vector into pieces and selecting top-$k$ elements from each piece generates different results from the top-$k$ elements on the full vector, which makes the proofs of the bounds of LAGS-SGD non-trivial. Recently, \cite{zheng2019communication} proposed the blockwise SGD for quantified gradients, but it lacks of convergence guarantees for sparsified gradients. Simultaneous to our work, Dutta et al. \cite{aritra2020discrepancy} proposed the layer-wise compression schemes and provided a different way of proof on the theoretical analysis.

\section{PRELIMINARIES}\label{sec:preliminaries}
We consider the common settings of distributed synchronous SGD with data-parallelism on $P$ workers to minimize the non-convex objective function $f: \mathbb{R}^d \to \mathbb{R}$ by:
\begin{equation}
    \mathbf{x}_{t+1}=\mathbf{x}_{t}-\alpha_t\frac{1}{P}\sum_{p=1}^{P}G^p(\mathbf{x}_t),
\end{equation}
where $\mathbf{x}_t\in \mathbb{R}^d$ is the stacked layer-wise model parameters of the target DNN at iteration $t$, $G^p(\mathbf{x}_t)$ is the stochastic gradients of the DNN parameters at the $p^{th}$ worker with locally sampled data, and $\alpha_t \in \mathbb{R}$ is the step size (i.e., learning rate) at iteration $t$. Let $L$ denote the number of learnable layers of the DNN, and $\mathbf{x}^{(l)} \in \mathbb{R}^{d^{(l)}}$ denote the parameter vector of the $l^{th}$ learnable layer with $d^{(l)}$ elements\footnote{This generalization is also applicable to the current deep learning frameworks (e.g., PyTorch), in which the parameters of one layer may be separated into two tensors (weights and bias).}. Thus, the model parameter $\mathbf{x}$ can be represented by the concatenation of $L$ layer-wise parameters. Using $\sqcup$ as the concatenation operator, the stacked vector can be represented by
\begin{equation}
    \mathbf{x}= \sqcup_{l=1}^{L} \mathbf{x}^{(l)}= \mathbf{x}^{(1)}\sqcup \mathbf{x}^{(2)}\sqcup...\sqcup\mathbf{x}^{(L)}=[\mathbf{x}^{(1)}, \mathbf{x}^{(2)}, ..., \mathbf{x}^{(L)}]. 
\end{equation}

\textbf{Pipelining between communications and computations}. Due to the fact that the gradient computation of layer $l-1$ using the backpropagation algorithm has no dependency on the gradient aggregation of layer $l$, the layer-wise communications can then be pipelined with layer-wise computations \cite{awan2017s,zhang2017poseidon} as shown in Fig. \ref{fig:algoprocesses}(a). It can be seen that some communication time can be overlapped with the computations so that the wall-clock iteration time is reduced. Note that the pipelining technique with full gradients has no side-effect on the convergence, and it becomes very useful when the communication time is comparable to the computing time.

\textbf{Top-$k$ sparsification}. In the gradient sparsification method, the Top-$k$ sparsification with error compensation \cite{aji2017sparse,lin2017deep,shi2019distributed} is promising in reducing communication traffic for distributed training, and its convergence property has been empirically \cite{aji2017sparse,lin2017deep} verified and theoretically \cite{alistarh2018convergence,jiang2018linear,stich2018sparsified} proved under some assumptions. The model update formula of Top-$k$ S-SGD can be represented by
\begin{equation}
    \mathbf{x}_{t+1}=\mathbf{x}_{t}-\alpha_t\frac{1}{P}\sum_{p=1}^{P}\widetilde{G}^p(\mathbf{x}_t),
\end{equation}
where $\widetilde{G}^{p}(\mathbf{x}_t)=\text{TopK}({G}^{p}(\mathbf{x}_t), k)$ is the selected top-$k$ gradients at worker $p$. For any vector $\mathbf{x}\in \mathbb{R}^d$ and a given $k \leq d$, $\text{TopK}(\mathbf{x}, k)\in \mathbb{R}^d$ and its $i^{th}$ ($i=1,2,...,d$) element is:
\begin{equation}
\text{TopK}(\mathbf{x}, k)_i=
\begin{cases}
x_i,&\text{if } |x_i|>thr\\
0, &\text{otherwise}
\end{cases},
\end{equation}
where $x_i$ is the $i^{th}$ element of $\mathbf{x}$ and $thr$ is the $k^{th}$ largest value of $|\mathbf{x}|$. As shown in Fig. \ref{fig:algoprocesses}(b), in each iteration, at the end of the backpropagation pass, each worker selects top-$k$ gradients from its whole set of gradients. The selected $k$ gradients are exchanged with all other workers in the decentralized architecture or sent to the parameter server in the centralized architecture for averaging.

\begin{figure*}[!h]
	\centering
	\includegraphics[width=\linewidth]{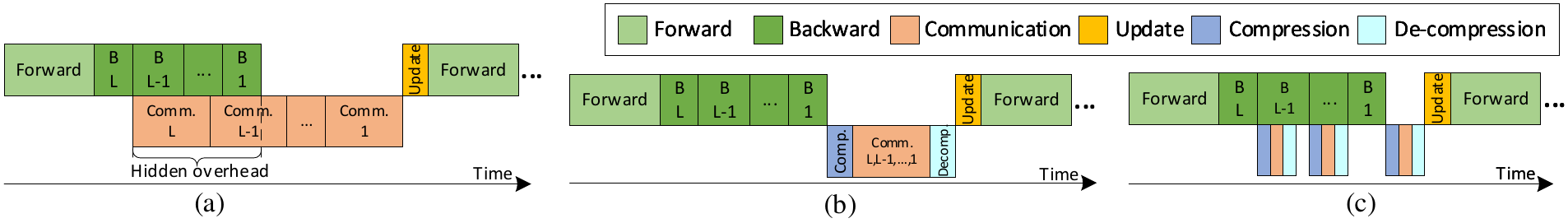}
	\vspace{-18pt}
	\caption{Comparison between three distributed training algorithms: (a) the pipeline of layer-wise gradient communications and backpropagation computations without gradient sparsification (Dense-SGD), (b) the single-layer gradient sparsification (SLGS) without pipelining, and (c) our proposed layer-wise adaptive gradient sparsification (LAGS) with pipelining.}
	\label{fig:algoprocesses}
\end{figure*}

\section{LAYER-WISE ADAPTIVE GRADIENT SPARSIFICATION}\label{sec:algo}
\subsection{Algorithm}
To enjoy the benefits of the pipelining technique and the gradient sparsification technique, we propose the LAGS-SGD algorithm, which exploits a layer-wise adaptive gradient sparsification (LAGS) scheme atop S-SGD. 

In LAGS-SGD, we apply gradient sparsification with error compensation on each layer separately. Instead of selecting the top-$k$ values from all gradients to be communicated, each worker selects top-$k^{(l)}$ gradients from layer $l$ so that it does not need to wait for the completion of backpropagation pass before communicating the sparsified gradients. LAGS-SGD not only significantly reduces the communication traffic (hence the communication time) using the gradient sparsification, but it also makes use of the layered structure of DNNs to overlap the communications with computations. As shown in Fig. \ref{fig:algoprocesses}(c), at each iteration, after the gradients $G(\mathbf{x})^{(l)}$ of layer $l$ have been calculated, $\text{TopK}(G(\mathbf{x})^{(l)}, k^{(l)})$ is selected to be exchanged among workers immediately. Formally, let $\mathbf{v}_t$ denote the model parameter and $\bm{\epsilon}^p_t$ denote the local gradient residual of worker $p$ at iteration $t$. In LAGS-SGD on distributed $P$ workers, the update formula of the $l$-layer's parameters becomes
\begin{equation}
    \mathbf{v}^{(l)}_{t+1}=\mathbf{v}^{(l)}_{t}-\frac{1}{P}\sum_{p=1}^{P}\text{TopK}\left(\alpha_t{G}^{p}(\mathbf{v}_t)^{(l)}+\bm{\epsilon}_{t}^{p,(l)}, k^{(l)}\right).
\end{equation}
The pseudo-code of LAGS-SGD is shown in Algorithm \ref{algo:lgsssgd}.
\begin{algorithm}[!h]
	\caption{LAGS-SGD at worker $p$}
	\label{algo:lgsssgd} 
	\small
	\textbf{Input: }Stochastic gradients $G^p(\cdot)$ at worker $p$\\
	\textbf{Input: }Configured $k^{(l)}$, $l=1,2,...,L$\\
	\textbf{Input: }Configured learning rates $\alpha_t$
	\begin{algorithmic}[1]
	    \For{$t=1\rightarrow L$}
	        \State Initialize $\mathbf{v}^{(l)}_0=\bm{\epsilon}^{p,(l)}_0=0$;
	    \EndFor
		\For{$t=1\rightarrow T$}
    		\State Feed-forward computation;
    		\For{$l=L\rightarrow 1$} 
    		    \State $acc_t^{p,(l)}=\bm{\epsilon}_{t-1}^{p,(l)}+\alpha_{t-1}  
    		    G^{p}(\mathbf{v}_{t-1})^{(l)}$; 
    		    \State $\bm{\epsilon}^{p,(l)}_{t}=acc_t^{p,(l)}-\text{TopK}(acc_t^{p,(l)}, k^{(l)})$; 
    		    \State $\mathbf{g}^{(l)}_t=\sum_{p=1}^{P}\text{TopK}(acc_t^{p,(l)}, k^{(l)})$; 
    	    	\State $\mathbf{v}^{(l)}_t=\mathbf{v}_{t-1}^{(l)}-\frac{1}{P}\mathbf{g}^{(l)}_t$; 
    		\EndFor
		\EndFor
	\end{algorithmic}
\end{algorithm}

\subsection{Convergence Analysis}\label{sec:convergence}
We first introduce some notations and assumptions for our convergence analysis, and then present the theoretical results of the convergence properties of LAGS-SGD.

\subsubsection{Notations and Assumptions}\label{subsec:assumptions}
Let $\|\cdot\|$ denote $\ell_2$-norm. We mainly discuss that the non-convex objective function $f:\mathbb{R}^d \to \mathbb{R}$ is $C$-smooth, i.e.,
\begin{equation}
    \|\nabla f(\mathbf{x})- \nabla f(\mathbf{y}) \| \leq C\|\mathbf{x}-\mathbf{y}\|, \forall \mathbf{x},\mathbf{y} \in \mathbb{R}^d.
\end{equation}
Let $\mathbf{x^*}$ denote the optimal solution of the objective function $f$. We assume that the sampled stochastic gradients $G(\cdot)$ are unbiased, i.e., $\mathbb{E}[G(\mathbf{v}_t)]=\nabla f(\mathbf{v}_t)$. We also assume that the second moment of the average of $P$ stochastic gradients has the following bound:
\begin{equation}\label{equ:varbound}
    \mathbb{E}[\|\frac{1}{P}\sum_{p=1}^{P}G^p(\mathbf{x})\|^2]\leq M^2, \forall \mathbf{x}\in \mathbb{R}^d.
\end{equation}
We make an analytical assumption on the aggregated results from the distributed sparsified vectors.
\begin{assumption} \label{assu:topk}
For any $P$ vectors $\mathbf{x}^p\in \mathbb{R}^d$ ($p=1,2,...P$) in $P$ workers, and each vector is sparsified as $\text{\normalfont TopK}(\mathbf{x}^p, k)$ locally. The aggregation of $\text{\normalfont TopK}(\mathbf{x}^p, k)$ selects $k$ larger values than randomly selecting $k$ values from the accumulated vectors, i.e.,
\begin{multline}
       \left\|\sum_{p=1}^{P}\mathbf{x}^p-\sum_{p=1}^{P}\text{\normalfont TopK}(\mathbf{x}^p,k)\right\|^2\leq \\
    \mathbb{E}\left[\left\|\sum_{p=1}^{P}\mathbf{x}^p-\text{\normalfont RandK}\left(\sum_{p=1}^{P}\mathbf{x}^p,k\right)\right\|^2\right],
\end{multline}
where $\text{\normalfont RandK}(\mathbf{x},k)\in \mathbb{R}^d$ is a vector whose $k$ elements are randomly selected from $\mathbf{x}$ following a uniform distribution, and the other $d-k$ elements are zeros.
\end{assumption}

Similar to \cite{alistarh2018convergence,shi2019aconvergence}, we introduce an auxiliary random variable $\mathbf{x}_t \in \mathbb{R}^d$, which is updated by the non-sparsified gradients, i.e.,
\begin{equation}\label{equ:auxiliary}
    \mathbf{x}_{t+1}=\mathbf{x}_t-\alpha_tG(\mathbf{v}_t),
\end{equation}
where $G(\mathbf{v}_t)=\frac{1}{P}\sum_{p=1}^{P}G^p(\mathbf{v}_t)$ and $\mathbf{x}_0=\mathbf{0}$. The error between $\mathbf{x}_t$ and $\mathbf{v}_t$ can be represented by
\begin{equation}
    \bm{\epsilon}_t=\mathbf{v}_t-\mathbf{x}_t=\frac{1}{P}\sum_{p=1}^{P}\bm{\epsilon}^p_t.
\end{equation}

\subsubsection{Main Results}
Here we present the major lemmas and theorems to prove the convergence of LAGS-SGD. Our results are mainly the derivation of the standard bounds in non-convex settings \cite{bottou2018optimization}, i.e.,
\begin{multline}
    \label{equ:target}
    \lim_{T\to \infty} \frac{1}{\sum_{t=1}^{T}\alpha_t}\sum_{t=1}^{T}\alpha_t\mathbb{E}[\|\nabla f(\mathbf{v}_t)\|^2] = 0, \\\text{ and } \mathbb{E}[\frac{1}{T}\sum_{t=1}^{T}\|\nabla f(\mathbf{v}_t)\|^2] \leq B,
\end{multline}
for some constants $B$ and the number of iterations $T$. 

\begin{lemma} \label{lemma:lwtopkbound}
For any $P$ vectors $\mathbf{x}^p \in \mathbb{R}^d, p=1,2,...,P$, and every vector can be broken down into $L$ pieces, that is $\mathbf{x}^p=\sqcup_{l=1}^{L} \mathbf{x}^{p,(l)}$ and $\mathbf{x}^{p,(l)}\in \mathbb{R}^{d^{(l)}}$, it holds that
\begin{align}\label{equ:lemma1}
    &\left\| \sum_{p=1}^{P} \mathbf{x}^p- \sqcup_{l=1}^{L} \left(\sum_{p=1}^{P} \text{\normalfont TopK}(\mathbf{x}^{p,(l)}, k^{(l)}) \right) \right\|^2 \notag \\ \leq &  (1-\frac{1}{c_{max}})\left\|\sum_{p=1}^{P}\mathbf{x}^p\right\|^2,
\end{align}
where $c_{max}=\max\{c^{(1)}, c^{(2)}, ..., c^{(L)}\}$, $c^{(l)}=\frac{d^{(l)}}{k^{(l)}}$ for $l=1,2,...,L$, $0< k^{(l)}\leq d^{(l)}$ and $\sum_{l=1}^{L}d^{(l)}=d$.
\end{lemma}
\begin{proof}
According to \cite{stich2018sparsified}, for any vectors $\mathbf{x} \in \mathbb{R}^d$ and $0 < k\leq d$, 
$\mathbb{E}_{\omega} [\left\| \mathbf{x} - \text{\normalfont  RandK}(\mathbf{x}, k) \right\|^2] =(1-\frac{k}{d})\| \mathbf{x} \|^2.$
Then 
\begin{align*}
    & \left\| \sum_{p=1}^{P} \mathbf{x}^p- \sqcup_{l=1}^{L} \left(\sum_{p=1}^{P} \text{\normalfont TopK}(\mathbf{x}^{p,(l)}, k^{(l)}) \right) \right\|^2  \\
    = &  \left\| \sqcup_{l=1}^{L} \left( \sum_{p=1}^{P} \mathbf{x}^{p,(l)}- \sum_{p=1}^{P} \text{\normalfont TopK}(\mathbf{x}^{p,(l)}, k^{(l)}) \right) \right\|^2  \\
    = & \sum_{l=1}^L \left\| \sum_{p=1}^{P} \mathbf{x}^{p,(l)}- \sum_{p=1}^{P} \text{\normalfont TopK}(\mathbf{x}^{p,(l)}, k^{(l)}) \right\|^2  \\
    % \leq & \mathbb{E}\left[ \sum_{l=1}^L \left\| \sum_{p=1}^{P} \mathbf{x}^{p,(l)}- \text{\normalfont RandK} \left(\sum_{p=1}^{P} \mathbf{x}^{p,(l)}, k^{(l)} \right) \right\|^2 \right] \\
    \leq & \sum_{l=1}^L \mathbb{E}\left[ \left\| \sum_{p=1}^{P} \mathbf{x}^{p,(l)}- \text{\normalfont RandK} \left(\sum_{p=1}^{P} \mathbf{x}^{p,(l)}, k^{(l)} \right) \right\|^2 \right] \\
    = & \sum_{l=1}^L \left(1-\frac{k^{(l)}}{d^{(l)}} \right) \left\| \sum_{p=1}^P \mathbf{x}^{p,(l)} \right\|^2 \\
    % = & \sum_{l=1}^L \left(1- \frac{1}{c^{(l)}} \right) \left\| \sum_{p=1}^P \mathbf{x}^{p,(l)} \right\|^2 \\
    \leq & \left( 1-\frac{1}{c_{max}} \right)\sum_{l=1}^L \left\| \sum_{p=1}^P \mathbf{x}^{p,(l)} \right\|^2 = \left( 1-\frac{1}{c_{max}} \right) \left\| \sum_{p=1}^P \mathbf{x}^{p} \right\|^2.
\end{align*}
\end{proof}
The inequality (\ref{equ:lemma1}) is a sufficient condition to derive the convergence properties of Algorithm \ref{algo:lgsssgd}.

\begin{corollary} \label{corollary:expbound}
For any iteration $t\geq 1$ and $\eta > 0$:
\begin{equation}
    \mathbb{E}[\| \mathbf{v}_t-\mathbf{x}_t \|^2]\leq \frac{1}{\eta}\sum_{i=1}^{t} \left( \left(1-\frac{1}{c_{max}} \right)(1+\eta) \right)^i \alpha_{t-i}^2M^2.
\end{equation}
\end{corollary}
\begin{proof} Let $G(\mathbf{v}_t)=\frac{1}{P}\sum_{p=1}^PG^p(\mathbf{v}_t)$, $\mathbf{g}_t=\sum_{p=1}^{P}(\alpha_t G^p(\mathbf{v}_t)+\bm{\epsilon}_t^p)$
and $\mathbf{g}_t^{(l)}=\sum_{p=1}^{P}(\alpha_t G^p(\mathbf{v}_t)^{(l)}+\bm{\epsilon}_t^{p,(l)}), l=1,2,...,L.$
We have $\mathbf{g}_t=\sqcup_{l=1}^{L} \mathbf{g}_t^{(l)}$ and $\bm{\epsilon}_t^{(l)}=\mathbf{g}_t^{(l)}-\sum_{p=1}^P \alpha_t G^p(\mathbf{v}_t)^{(l)}$. According to the update formulas of $\mathbf{v}_{t+1}$ and $\mathbf{x}_{t+1}$, we have
\begin{align*} 
  &\|\mathbf{v}_{t+1} - \mathbf{x}_{t+1} \|^2 \\
 = & \| \sqcup_{l=1}^{L} ( \mathbf{v}_t^{(l)} - \frac{1}{P} \sum_{p=1}^{P} \text{TopK}(\mathbf{g}^{p,(l)}, k^{(l)}) \\
 &- \mathbf{x}_t^{(l)} + \frac{1}{P} \sum_{p=1}^P \alpha_t G^p(\mathbf{v}_t)^{(l)} ) \|^2 \\
  = & \sum_{l=1}^{L}\left\| \frac{1}{P}\sum_{p=1}^P\mathbf{g}_t^{p,(l)} - \frac{1}{P} \sum_{p=1}^P \text{TopK}\left(\mathbf{g}^{p,(l)}, k^{(l)} \right) \right\|^2 \\
  = & \left\| \sqcup_{l=1}^{L} \left( \frac{1}{P}\sum_{p=1}^P\mathbf{g}_t^{p,(l)} - \frac{1}{P} \sum_{p=1}^P \text{TopK}\left(\mathbf{g}^{p,(l)}, k^{(l)} \right) \right) \right\|^2 \\
  = & \left\| \frac{1}{P}\sum_{p=1}^P\mathbf{g}_t^{p} - \frac{1}{P}\sqcup_{l=1}^{L} \left(  \sum_{p=1}^P \text{TopK}\left(\mathbf{g}^{p,(l)}, k^{(l)} \right) \right) \right\|^2 \\
  \leq & (1-\frac{1}{c_{max}})\left\| \frac{1}{P} \sum_{p=1}^{P}\mathbf{g}_t^p \right\|^2 = (1-\frac{1}{c_{max}})\left\| \alpha_tG(\mathbf{v}_t)+\mathbf{v}_t-\mathbf{x}_t \right\|^2 \\
  \leq & (1-\frac{1}{c_{max}})\left( (1+\eta)\|\alpha_tG(\mathbf{v}_t)\|^2 + (1+\frac{1}{\eta})\| \mathbf{v}_t-\mathbf{x}_t \|^2 \right),
\end{align*} 
where $\eta > 0$. Iterating the above inequality from $i=0\to t$ yields:
\begin{align*}
&\| \mathbf{v}_t-\mathbf{x}_t \|^2 \\
\leq& (1-\frac{1}{c_{max}})(1+\frac{1}{\eta})\sum_{i=1}^{t}((1-\frac{1}{c_{max}}) (1+\eta))^{i-1}\|\alpha_{t-i}G(\mathbf{v}_{t-i})\|^2 \\
= & \frac{1}{\eta}\sum_{i=1}^{t}((1-\frac{1}{c_{max}}) (1+\eta))^{i}\|\alpha_{t-i}G(\mathbf{v}_{t-i})\|^2.
\end{align*}
Taking the expectation and using the bound of the second moment on stochastic gradients: $\mathbb{E}[\|G(\mathbf{v}_t)\|^2]\leq M^2$, we obtain
\begin{align*}
\mathbb{E}\left[\| \mathbf{v}_t-\mathbf{x}_t \|^2\right] 
\leq & \frac{1}{\eta}\sum_{i=1}^{t}\left((1-\frac{1}{c_{max}}) (1+\eta)\right)^{i}\mathbb{E}\left[\|\alpha_{t-i}G(\mathbf{v}_{t-i})\|^2\right]\\
\leq & \frac{1}{\eta}\sum_{i=1}^{t}\left((1-\frac{1}{c_{max}}) (1+\eta)\right)^{i}\alpha_{t-i}^2M^2,
\end{align*}
which concludes the proof.
\end{proof}
Corollary \ref{corollary:expbound} implies that the parameters with sparsified layer-wise gradients have bounds compared to that with dense gradients. 

\begin{theorem1}\label{theo:main}
Under the assumptions defined in the objective function $f$, after running $T$ iterations with Algorithm \ref{algo:lgsssgd}, we have
\begin{align}\label{equ:main}
    &\frac{1}{\sum_{t=1}^T\alpha_t}\sum_{t=1}^T\alpha_t\mathbb{E}[\|\nabla f(\mathbf{v}_t)\|^2] \notag \\\leq &  \frac{4(f(\mathbf{x}_0)-f(\mathbf{x}^*))}{\sum_{t=1}^T\alpha_t}+ 
    \frac{2(C+\frac{2C^2D}{\eta})M^2\sum_{t=1}^T\alpha_t^2}{\sum_{t=1}^T\alpha_t},
\end{align}
if one chooses a step size schedule such that $\exists D>0$ and $\exists \eta >0$,
\begin{equation}\label{equ:lrbound}
    \sum_{i=1}^{t}\left(\left(1-\frac{1}{c_{max}}\right)(1+\eta)\right)^{i}\frac{\alpha_{t-i}^2}{\alpha_t}\leq D
\end{equation}
holds at any iteration $t> 0$. 
\end{theorem1}

\begin{proof}
We use the smooth property of $f$ and Corollary 1 to derive (\ref{equ:main}). First, with the smoothness $C$ of $f$, we have
\begin{align*}
    f(\mathbf{x}_{t+1})-f(\mathbf{x}_{t}) \leq & \nabla f(\mathbf{x}_t)^T(\mathbf{x}_{t+1}-\mathbf{x}_{t}) + \frac{C}{2}\| \mathbf{x}_{t+1}-\mathbf{x}_{t} \|^2 \\
    = & \alpha_t \nabla f(\mathbf{x}_t)^TG(\mathbf{v}_t) + \frac{\alpha_t^2 C}{2}\| G(\mathbf{v}_t)\|^2.
\end{align*}
Taking expectation with respective to sampling at $t$, it yields
\begin{align*}
     & \mathbb{E}_t[f(\mathbf{x}_{t+1})]-f(\mathbf{x}_{t}) \\
    \leq & \alpha_t \nabla f(\mathbf{x}_t)^T \mathbb{E}_t[G(\mathbf{v}_t)] + \frac{\alpha_t^2 C}{2} \mathbb{E}_t[\| G(\mathbf{v}_t)\|^2] \\
    = & \alpha_t \nabla f(\mathbf{x}_t)^T \nabla f(\mathbf{v}_t) + \frac{\alpha_t^2 C}{2} \mathbb{E}_t[\| G(\mathbf{v}_t)\|^2] \\
    =&-\frac{\alpha_t}{2}\|\nabla f(\mathbf{x}_t)\|^2-\frac{\alpha_t}{2}\|\nabla f(\mathbf{v}_t)\|^2 \\
    &+\frac{\alpha_t}{2}\|\nabla f(\mathbf{x}_t) -\nabla f(\mathbf{v}_t)\|^2+\frac{\alpha_t^2C}{2}\mathbb{E}_t[\|G(\mathbf{v}_t)\|^2] \\
    \leq & -\frac{\alpha_t}{2}\|\nabla f(\mathbf{x}_t)\|^2 + \frac{\alpha_t C^2}{2}\|\mathbf{v}_t-\mathbf{x}_t\|^2 +\frac{\alpha_t^2CM^2}{2} \\
    = & -\frac{\alpha_t}{2}(\|\nabla f(\mathbf{x}_t)\|^2 + C^2\|\mathbf{v}_t-\mathbf{x}_t\|^2) \\
    &+\alpha_tC^2\|\mathbf{v}_t-\mathbf{x}_t\|^2 +\frac{\alpha_t^2CM^2}{2}. 
\end{align*}
Taking expectation with respective to the gradients before $t$, it yields
\begin{align*}
 \mathbb{E}[f(\mathbf{x}_{t+1})]-\mathbb{E}[f(\mathbf{x}_t)] 
 \leq & -\frac{\alpha_t}{2}\mathbb{E}[\|\nabla f(\mathbf{x}_t)\|^2+ C^2\|\mathbf{v}_t-\mathbf{x}_t\|^2] \\ &+\alpha_tC^2\mathbb{E}[\|\mathbf{v}_t-\mathbf{x}_t\|^2] +\frac{\alpha_t^2CM^2}{2}.
\end{align*}
Using Corollary 1, we obtain
\begin{align*}
      & \mathbb{E}[f(\mathbf{x}_{t+1})]-\mathbb{E}[f(\mathbf{x}_t)] \\
    \leq & \frac{\alpha_tC^2}{\eta}\sum_{i=1}^{t}((1-\frac{1}{c_{max}})(1+\eta))^{i}\alpha_{t-i}^2 M^2+\frac{\alpha_t^2CM^2}{2} \\ 
    & -\frac{\alpha_t}{2}\mathbb{E}[\|\nabla f(\mathbf{x}_t)\|^2+ C^2\|\mathbf{v}_t-\mathbf{x}_t\|^2] \\
    =&\frac{\alpha_t^2 C^2}{\eta}\sum_{i=1}^{t}((1-\frac{1}{c_{max}})(1+\eta))^{i}\frac{\alpha_{t-i}^2}{\alpha_{t}}M^2+\frac{\alpha_t^2CM^2}{2} \\
    &-\frac{\alpha_t}{2}\mathbb{E}[\|\nabla f(\mathbf{x}_t)\|^2+ C^2\|\mathbf{v}_t-\mathbf{x}_t\|^2].
\end{align*}
If (\ref{equ:lrbound}) holds, then we have
\begin{align*}
&\mathbb{E}[f(\mathbf{x}_{t+1})]-\mathbb{E}[f(\mathbf{x}_t)] \\
\leq & (C+\frac{2C^2D}{\eta})\frac{M^2 \alpha_t^2}{2}-\frac{\alpha_t}{2}\mathbb{E}[\|\nabla f(\mathbf{x}_t)\|^2+ C^2\|\mathbf{v}_t-\mathbf{x}_t\|^2].
\end{align*}
Adjusting the order, we obtain
\begin{align}\label{in:1}
&\alpha_t\mathbb{E}[\|\nabla f(\mathbf{x}_t)\|^2+ C^2\|\mathbf{v}_t-\mathbf{x}_t\|^2] \notag\\
\leq & 2(\mathbb{E}[f(\mathbf{x}_t)]  - \mathbb{E}[f(\mathbf{x}_{t+1})]) +(C+\frac{2C^2D}{\eta})M^2 \alpha_t^2.
\end{align}
We further apply the property of $f$, that is
\begin{align*}
    \|\nabla f(\mathbf{v}_t)\|^2 &=\|\nabla f(\mathbf{v}_t)-\nabla f(\mathbf{x}_t)+\nabla f(\mathbf{x}_t)\|^2\\
    &\leq 2\|\nabla f(\mathbf{v}_t)-\nabla f(\mathbf{x}_t)\|^2+2\|\nabla f(\mathbf{x}_t)\|^2\\
    &\leq 2C^2\|\mathbf{v}_t-\mathbf{x}_t\|^2+2\|\nabla f(\mathbf{x}_t)\|^2.
\end{align*}
Together with (\ref{in:1}), it yields
\begin{align*}
    &\alpha_t\mathbb{E}[\|\nabla f(\mathbf{v}_t)\|^2] 
    \leq  2\alpha_t \mathbb{E}[C^2\|\mathbf{v}_t-\mathbf{x}_t\|^2+\|\nabla f(\mathbf{x}_t)\|^2]\\
    \leq& 4(\mathbb{E}[f(\mathbf{x}_{t})]-\mathbb{E}[f(\mathbf{x}_{t+1})])+ 2(C+\frac{2C^2D}{\eta})M^2 \alpha_t^2.
\end{align*}
Summing up the inequality for $t=1,2,...,T$, it yields
\begin{align*}
    \sum_{t=1}^{T}\alpha_t\mathbb{E}[\|\nabla f(\mathbf{v}_t)\|^2] 
    \leq&  4(f(\mathbf{x}_0)-f(\mathbf{x}^*)) \\
    &+2(C+\frac{2C^2D}{\eta})M^2\sum_{t=1}^{T}\alpha_t^2.
\end{align*}
Multiplying $\frac{1}{\sum_{t=1}^{T}\alpha_t}$ in both sides concludes the proof.
\end{proof}

Theorem \ref{theo:main} indicates that if one chooses the step sizes to satisfy (\ref{equ:lrbound}), then the right hand side of (\ref{equ:main}) converges as $T\to \infty$, so that Algorithm \ref{algo:lgsssgd} is guaranteed to converge. If we let $(1-\frac{1}{c_{max}})(1+\eta) < 1$ which is easily satisfied then (\ref{equ:lrbound}) holds for both constant and diminishing step sizes. Therefore, if the step sizes are further configured as 
\begin{equation}
    \lim_{T\to \infty}\sum_{t=1}^T\alpha_t=\infty \text{ and } \lim_{T\to \infty}\sum_{t=1}^T\alpha^2_t<\infty,
\end{equation}
then the right hand side of inequality (\ref{equ:main}) converges to zero, which ensures the convergence of Algorithm \ref{algo:lgsssgd}. 

\begin{corollary}\label{corollary:convergencerate} 
Under the same assumptions in Theorem \ref{theo:main}, if $\alpha_t=\theta/ \sqrt{T}, \forall t>0$, where $\theta > 0$ is a constant, then we have the convergence rate bound for Algorithm \ref{algo:lgsssgd} as:
\begin{multline}\label{equ:final}
    \mathbb{E}[\frac{1}{T}\sum_{t=1}^{T}\|\nabla f(\mathbf{v}_t)\|^2]
        \leq \frac{4\theta^{-1}(f(\mathbf{x}_0)-f(\mathbf{x}^*))+ 2\theta CM^2}{\sqrt{T}}\\ +\frac{4C^2M^2(c_{max}^3-c_{max})\theta^2}{T},
\end{multline}
if the total number of iterations $T$ is large enough.
\end{corollary}
\begin{proof}
As $\alpha_t=\theta/ \sqrt{T}$, we simplify the notations by: $\alpha=\alpha_t=\theta/ \sqrt{T}$ and $\tau=(1-\frac{1}{c_{max}})(1+\eta)$. The left hand side of (\ref{equ:lrbound}) becomes
\begin{align*}
    \sum_{i=1}^{t}((1-\frac{1}{c_{max}})(1+\eta))^{i}\frac{\alpha_{t-i}^2}{\alpha_t}
    =\sum_{i=1}^{t}\tau^{i}\frac{\alpha_{t-i}^2}{\alpha_t}
    =\alpha \frac{\tau(1-\tau^t)}{1-\tau}.
\end{align*}
Let $\eta=\frac{1}{c_{max}}$, then $0\leq \tau=(1-\frac{1}{c_{max}})(1+\eta)<1$, and
\begin{align*}
    \lim_{t\to \infty}\alpha \frac{ \tau(1-\tau^t)}{1-\tau}=\frac{\alpha \tau}{1-\tau}.
\end{align*}
So (\ref{equ:lrbound}) holds when $D=\frac{\alpha \tau}{1-\tau}$. Applying Theorem \ref{theo:main}, we obtain
\begin{align*}
    &\mathbb{E}\left[\frac{1}{T}\sum_{t=1}^{T}\|\nabla f(\mathbf{v}_t)\|^2\right] \\ \leq&  \frac{4(f(\mathbf{x}_0)-f(\mathbf{x}^*))}{\alpha T}+2(C+\frac{2C^2D}{\eta})M^2\alpha \\
    = & \frac{4\theta^{-1}(f(\mathbf{x}_0)-f(\mathbf{x}^*))+2\theta CM^2}{\sqrt{T}}+\frac{4C^2M^2(c_{max}^3-c_{max})\theta^2}{T},
\end{align*}
which concludes the proof.
\end{proof}

In Corollary \ref{corollary:convergencerate}, if $T$ is large enough, then the right hand side of inequality (\ref{equ:final}) is dominated by the first term. It implies that Algorithm \ref{algo:lgsssgd} has a convergence rate of $O(1/\sqrt{T})$, which is the same as the vanilla SGD \cite{dekel2012optimal}. However, the second term of inequality (\ref{equ:final}) also indicates that higher compress ratios (i.e., $c_{max}$) lead to a larger bound of the convergence rate. In real-world settings, one may have a fixed number of iteration budget $T$ to train the model, so high compression ratios could slowdown the convergence speed. On the one hand, if we choose lower compression ratios, then the algorithm has a faster convergence rate (less number of iterations). On the other hand, lower compression ratios have a larger communication size and thus may result in longer wall-clock time per iteration. Therefore, the adaptive selection of the compression ratios tackles the problem properly.

\section{SYSTEM IMPLEMENTATION AND OPTIMIZATION}\label{sec:system}
The layer-wise sparsification nature of LAGS-SGD enables the pipelining technique to hide the communication overheads, while the efficient system implementation of communication and computation parallelism with gradient sparsification is non-trivial due to three reasons: 1) Layer-wise communications with sparsified gradients indicate that there exist many small size messages to be communicated across the network, while collectives (e.g., AllReduce) with small messages are latency-sensitive. 2) Gradient sparsification (especially top-$k$ selection on GPUs) would introduce extra computation time. 3) The convergence rate of LAGS-SGD is negatively effected by the compression ratio, and one should decide proper compression ratios to trade-off the number of iteration to converge and the iteration wall-clock time. 

First, we exploit a heuristic method to merge extremely small sparsified tensors to a single one for efficient communication to address the first problem. Specifically, we use a memory buffer to temporarily store the sparsified gradients, and aggregate the buffered gradients once the buffer becomes full or the gradients of the first layer have been calculated. Second, we implement the double sampling method \cite{lin2017deep} to approximately select the top-$k$ gradients, which can significantly reduce the top-$k$ selection time on GPUs. Finally, to achieve a balance between the convergence rate and the training wall-clock time, we propose to select the layer-wise compression ratio according to the communication-to-computation ratio. To be specific, we select a compression ratio $c^{(l)}=d^{(l)}/k^{(l)}$ for layer $l$ such that its communication overhead is appropriately hidden by the computation. Given an upper bound of the compression ratio (e.g., $c_u=1000$), the algorithm determines $c^{(l)}$ according to the following three metrics: 1) Backpropagation computation time of the pipelined layers (i.e., $t_{comp}^{(l-1)}$); 2) Communication time of the current layer $t_{comm}^{(l)}$ under a specific compression ratio $c^{(l)}$, which can be predicted using the communication model of the AllGather or AllReduce collectives (e.g., \cite{li2018pipe,renggli2018sparcml}) according to the size of gradients and the inter-connection (e.g., latency and bandwidth) between workers; 3) An extra overhead involved by the sparsification operator ($t_{spar}^{(l)}$), which generally includes a pair of operations (compression and de-compression). Therefore, the selected value of $c^{(l)}$ can be generalized as
\begin{equation}\label{equ:adapratio}
    c^{(l)}= \max \{c_u, \min { \{ c | t_{comm}^{(l)}(c)+t_{spar}^{(l)} \leq t_{comp}^{(l-1)}\} }\}.
\end{equation}

\subsection{Bound of Pipelining Speedup}
In LAGS-SGD, the sparsification technique is used to reduce the overall communication time, and the pipelining technique is used to further overlap the already reduced communication time with computation time. We can analyze the optimal speedup of LAGS-SGD over SLGS-SGD in terms of wall-clock time under the same compression ratios. Let $t_{f}$, $t_{b}$ and $t_{c}$ denote the forward computation, backward computation and gradient communication time at each iteration respectively. We assume that the sparsification overhead can be ignored as we can use the efficient sampling method. Compared to SLGS-SGD, LAGS-SGD reduces the wall-clock time by pipelining the communications with computations, and the maximum overlapped time is $t_{hidden}=\min\{t_{b}, t_{c}\}$ (i.e., either backpropagation computations or communications are completely overlapped). So the maximum speedup of LAGS-SGD over SLGS-SGD can be calculated as $S=(t_{f}+t_{b}+t_{c})/(t_{f}+t_{b}+t_{c}-t_{hidden})$. Let $r=t_{c}/t_{b}$ denote the communication-to-computation ratio. The ideal speedup can be represented by
\begin{equation}\label{equ:smax}
    S_{max}= 1 + \frac{1}{\frac{t_{f}}{\min(t_c,t_b)}+\max(r,1/r)}.
\end{equation}
The equation shows that the maximum speedup of LAGS-SGD over SLGS-SGD mainly depends on the communication-to-computation ratio. If $r$ is close to $1$, then LAGS-SGD has the potential to achieve the highest speedup by completely hiding either the backpropagation computation or the communication time.

\section{EXPERIMENTS}\label{sec:experiments}

\subsection{Experimental Settings}
We conduct the similar experiments as the work \cite{lin2017deep}, which cover two types of applications with three data sets: 1) image classification by convolutional neural networks (CNNs) such as ResNet-20 \cite{he2016deep} and VGG-16 \cite{simonyan2014very} on the Cifar-10 \cite{krizhevsky2010cifar} data set and Inception-v4 \cite{szegedy2017inception} and ResNet-50 \cite{he2016deep} on the ImageNet \cite{deng2009imagenet} data set; 2) language model by a 2-layer LSTM model (LSTM-PTB) with 1500 hidden units per layer on the PTB \cite{marcus1993building} data set. On Cifar-10, the batch size for each worker is 32, and the base learning rate is 0.1; On ImageNet, the batch size for each worker is also 32, and the learning rate is 0.01; On PTB, the batch size and learning rate is 20 and 22 respectively. We set the compression ratios as $1000$ and $250$ for CNNs and LSTM respectively. In all compared algorithms, the hyper-parameters are kept the same and experiments are conducted on a 16-GPU cluster.

\begin{figure*}[!ht]
	\centering
	\includegraphics[width=0.95\linewidth]{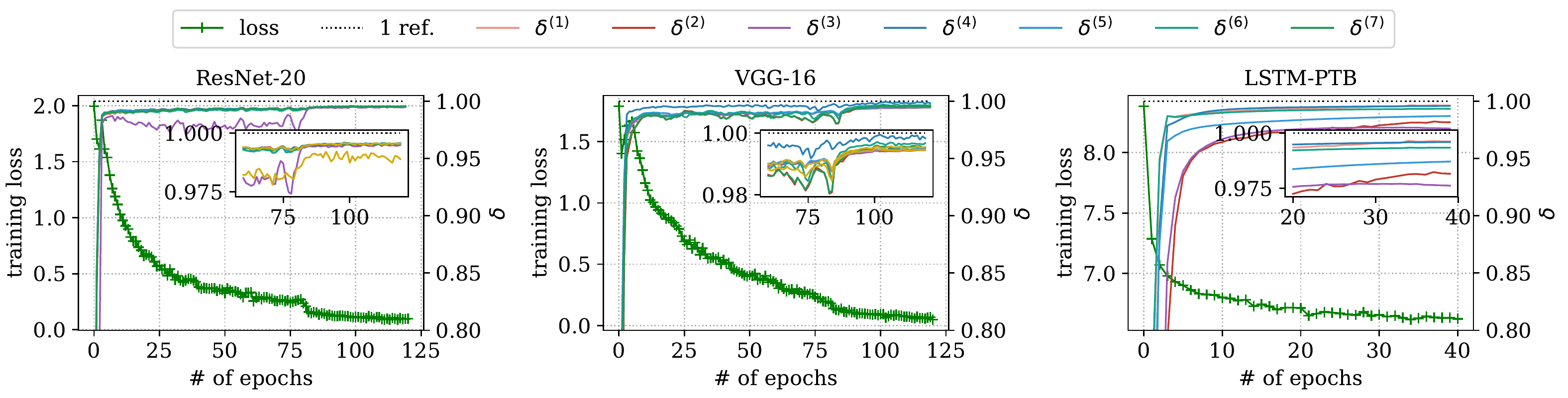}
	\vspace{-10pt}
	\caption{The values of $\delta^{(l)}$ ($7$ layers are displayed for better visualization), and the training loss of LAGS-SGD on 16 workers.}
	\label{fig:verifyass}
\end{figure*}
\subsection{Verification of Assumption \ref{assu:topk} and Convergences}
To show the soundness of Assumption \ref{assu:topk} and the convergence results, we conduct the experiments with 16 workers to train the models. We define metrics $\delta^{(l)}$ ($l=1,2,...,L$) for each learnable layer during the training process at each iteration with Algorithm \ref{algo:lgsssgd}, and
\begin{equation}
\delta^{(l)}= \frac{\left\|\sum_{p=1}^{P}\mathbf{x}^{p,(l)}-\sum_{p=1}^{P}\text{\normalfont TopK}(\mathbf{x}^{p,(l)},k^{(l)})\right\|^2}{
    \left\|\sum_{p=1}^{P}\mathbf{x}^{p,(l)}-\text{\normalfont RandK}\left(\sum_{p=1}^{P}\mathbf{x}^{p,(l)},k^{(l)}\right)\right\|^2},
\end{equation}
where $\mathbf{x}^{p,(l)}=G^p(\mathbf{v}_t)^{(l)}+\bm{\epsilon}_t^{p,(l)}$. Assumption \ref{assu:topk} holds if $\delta^{(l)} \leq 1$ ($l=1,2,...,L$). We measure $\delta^{(l)}$ on ResNet-20, VGG-16 and LSTM-PTB during training, and the results are shown in Fig. \ref{fig:verifyass}. It is seen that $\delta^{(l)}<1$ throughout the training process, which implies that Assumption \ref{assu:topk} holds. The evaluated models all converge in a certain number of epochs, which verifies the convergence of LAGS-SGD.

\subsection{Comparison of Convergence Rates}
\begin{figure}[!ht]
	\centering
	\includegraphics[width=0.98\linewidth]{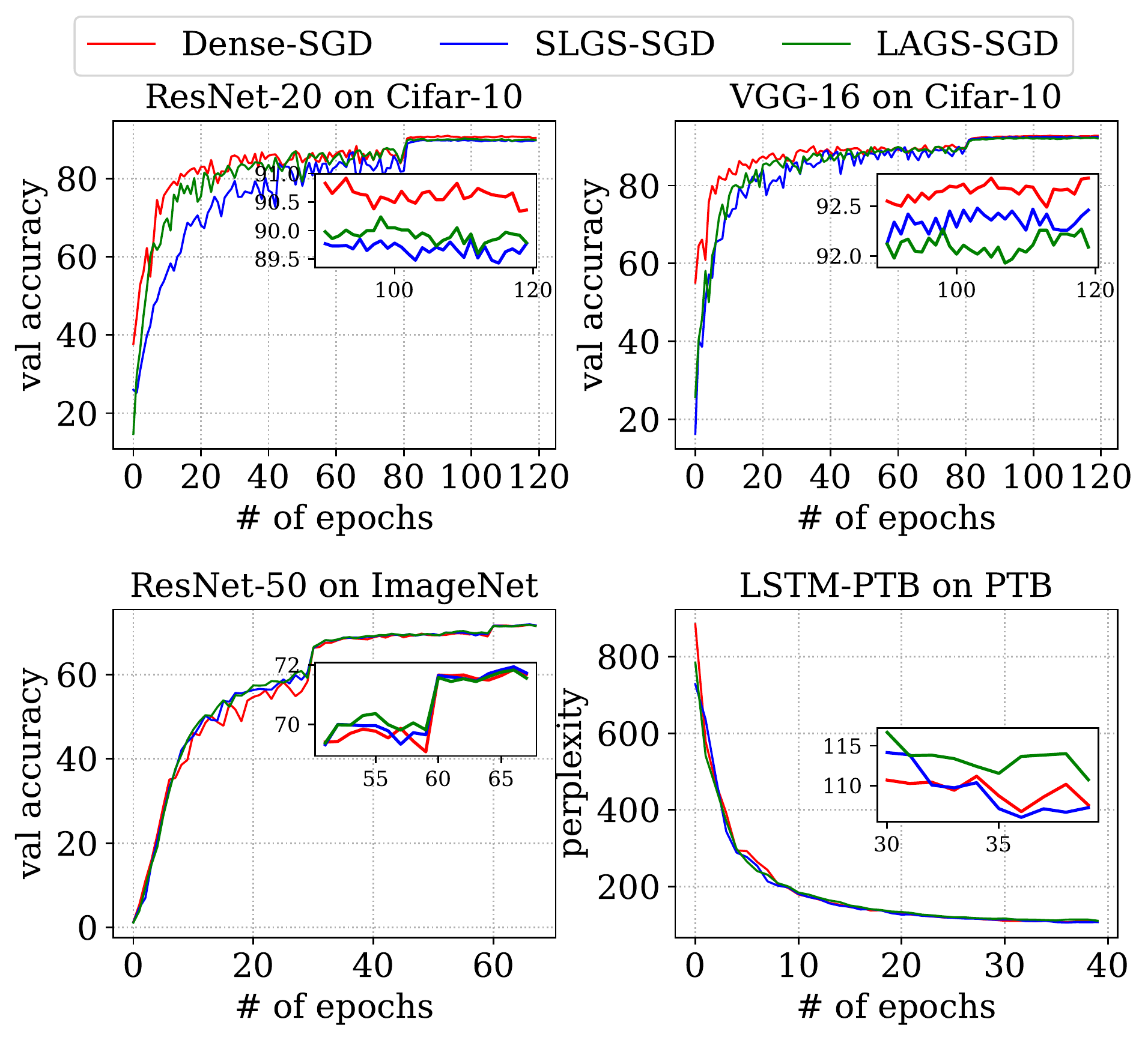}
	\vspace{-8pt}
	\caption{The comparison of convergence performance.}
	\label{fig:convergencerates}
\end{figure}
The convergence comparison under the same number of training epochs is shown in Fig. \ref{fig:convergencerates}. The top-1 validation accuracy (the higher the better) on CNNs and the validation perplexity (the lower the better) on LSTM show that LAGS-SGD has very close convergence performance to SLGS-SGD. Compared to Dense-SGD, SLGS-SGD and LAGS-SGD both have slight accuracy losses. The problem could be resolved by some training tricks like momentum correction \cite{lin2017deep}. The final evaluation results are shown in Table \ref{table:acc}. The nearly consistent convergence performance between LAGS-SGD and Dense-SGD verifies our theoretical results on the convergence rate.
\begin{table}
    \centering
    \caption{Comparison of evaluation performance. Top-1 validation accuracy for CNNs and perplexity for LSTM-PTB.}\label{table:acc}
    \begin{tabular}{|c|c|c|c|}
    \hline
       Model  & Dense-SGD & SLGS-SGD  & LAGS-SGD  \\\hline\hline
       ResNet-20  & $0.9092$ & $0.8985$ & $0.9024$ \\\hline
       VGG-16  & $0.9278$ & $0.9255$ & $0.9227$ \\\hline
       ResNet-50  & $0.7191$ & $0.7211$ & $0.7183$ \\\hline
       LSTM-PTB & $106.7$ & $105.7$ & $109.4$ \\\hline
    \end{tabular}
    \vspace{-8pt}
\end{table}

\subsection{Wall-clock Time Performance and Discussions}
We evaluate the average iteration time with CNNs including VGG-16I (VGG-16 \cite{simonyan2014very} for ImageNet), ResNet-50 and Inception-v4 on the large-scale data set ImageNet (over one million training images) on a 16-GPU cluster (four nodes and each node contains four Nvidia Tesla V100 PCIE-32G GPUs) connected with 10Gbps Ethernet (10GbE). The servers in the cluster are with Intel CPUs (Xeon E5-2698v3 Dual), Ubuntu-16.04 and CUDA-10.0. The main libraries used in our experiments are PyTorch-v1.1, OpenMPI-v4.0.0, Horovod-v0.18.1 and NCCL-v2.3.7. The experimental results are shown in Table \ref{table:results} using the compression ratio of $1000$ for gradient sparsification. It demonstrates that LAGS-SGD performs around $3.3\%-4.5\%$ faster than SLGS-SGD, which is up to $45\%$ close to the maximum speedup, and it achieves $22\%-30\%$ improvement over Dense-SGD. 

\begin{table}
\vspace{-4pt}
    \centering
     \caption{The average iteration time in seconds of 1000 running iterations. $S_1$ and $S_2$ indicate the speedups of LAGS-SGD over Dense-SGD and SLGS-SGD respectively. $S_{max}$ is the maximum speedup of pipelining over SLGS-SGD.}
    \label{table:results}
    \centering
    \addtolength{\tabcolsep}{-1.pt}
    \begin{tabular}{|l|c|c|c||c|c|c|}
    \hline
         Model & Dense & SLGS & LAGS & $S_1$ & $S_2$ & $S_{max}$\\\hline\hline
        VGG-16I & $1.488s$ & $1.003$ & $0.961$ & $1.507$ & $1.044$ & $1.096$  \\\hline
        ResNet-50 & $0.570s$ & $0.485s$  & $0.464s$ & $1.228$ & $1.045$ & $1.133$ \\\hline
        Inception-v4 & $0.891s$ & $0.749s$ & $0.725s$ & $1.229$ & $1.033$ & $1.204$ \\\hline
    \end{tabular}
    % \vspace{-8pt}
\end{table}

The achieved speedups of LAGS-SGD over SLGS-SGD in the end-to-end training wall-clock time are minor, which is caused by three main reasons. First, as shown in Eq. (\ref{equ:smax}), the improvement of LAGS-SGD over SLGS-SGD is highly depended on the communication-to-computation ratio $r$. In the conducted experiments, $r$ is small because transferring highly sparsified data under 10GbE is much faster than the computation time on Nvidia Tesla V100 GPUs, while the proposed method has potential improvement with increased $r$ such as lower bandwidth networks. Second, the compression time is not negligible compared to the communication time. Even we exploit the sampling method \cite{lin2017deep} to select the top-$k$ gradients, it is inefficient on GPUs so that it enlarges the computation time, which makes $r$ smaller. For example, in the VGG-16I model, the backpropagation time is around $0.391s$, while the gradient compression time and the communication time are about $0.359s$ and $0.049s$ respectively. It is worthy to explore more efficient selection algorithms for gradient sparsification like \cite{shi2019understanding}. Third, the layer-wise communications introduce many startup times in transferring small tensors which could make the performance even worse if the communications are not scheduled properly. It could be possible to exploit the adaptive tensor fusion method \cite{shi2019mgwfbp} to further improve the scalability. We will leave this as our future work.  

\section{CONCLUSION}\label{sec:conclude}
In this paper, we proposed a new distributed optimization algorithm for deep learning named LAGS-SGD, which exploits a novel layer-wise adaptive gradient sparsification scheme to embrace the promising pipelining techniques and gradient sparsification methods. LAGS-SGD not only takes advantage of the gradient sparsification algorithm to reduce the communication size, but also makes use of the pipelining technique to further hide the communication overhead. We provided detailed theoretical analysis for LAGS-SGD which showed that LAGS-SGD has convergence guarantees and the consistent convergence rate as the original SGD under a weak analytical assumption. We ran extensive experiments to verify the soundness of the analytical assumption and theoretical results. Experimental results on a 16-node GPU cluster connected with 10Gbps Ethernet interconnect demonstrated that LAGS-SGD outperforms the state-of-the-art sparsified S-SGD and Dense-SGD with comparable model accuracy.

\label{ack}
\ack The research was supported by Hong Kong RGC GRF grant HKBU 12200418. We acknowledge Nvidia AI Technology Centre (NVAITC) for providing GPU clusters for experiments.

%\bibliography{cites-short}
\bibliography{layerwise-sparsification.bbl}

\end{document}